\pdfoutput=1
\documentclass[letterpaper, 10 pt, conference]{cls/ieeeconf}

\IEEEoverridecommandlockouts  
\let\labelindent\relax

\usepackage[table,usenames,dvipsnames]{xcolor}
\usepackage{enumitem}
\usepackage{cite}
\usepackage{amsmath,amssymb,amsfonts,amsthm,amscd,dsfont}
\usepackage{accents} 
\usepackage{algorithm,listings}
\usepackage[noend]{algpseudocode}
\usepackage{graphicx,tabularx,adjustbox}
\usepackage[font={small}]{caption}
\usepackage{subcaption}
\usepackage{paralist}

\let\appendices\relax
\usepackage[titletoc,title]{appendix}
\usepackage{balance}
\usepackage[breaklinks=true, colorlinks, bookmarks=true, citecolor=Black, urlcolor=Violet,linkcolor=Black]{hyperref}

\newcommand{\prl}[1]{\left(#1\right)}

\newcommand{\crl}[1]{\left\{#1\right\}}

\newcommand{\scaleMathLine}[2][1]{\resizebox{#1\linewidth}{!}{$\displaystyle{#2}$}}

\def\negquad{\mkern-18mu}

\newtheorem{proposition}{Proposition}

\theoremstyle{definition}
\newtheorem{definition}{Definition}

\newtheorem*{problem}{Problem}
\theoremstyle{remark}

\newcommand{\calE}{{\cal E}}
\newcommand{\calF}{{\cal F}}

\newcommand{\calI}{{\cal I}}

\newcommand{\calK}{{\cal K}}

\newcommand{\calN}{{\cal N}}

\newcommand{\calR}{{\cal R}}

\newcommand{\calU}{{\cal U}}

\newcommand{\calZ}{{\cal Z}}

\newcommand{\bfe}{\mathbf{e}}

\newcommand{\bfh}{\mathbf{h}}

\newcommand{\bfl}{\mathbf{l}}
\newcommand{\bfm}{\mathbf{m}}

\newcommand{\bfp}{\mathbf{p}}

\newcommand{\bfu}{\mathbf{u}}
\newcommand{\bfv}{\mathbf{v}}

\newcommand{\bfz}{\mathbf{z}}

\newcommand{\bfeta}{\boldsymbol{\eta}}

\newcommand{\bfphi}{\boldsymbol{\phi}}

\newcommand{\bfpsi}{\boldsymbol{\psi}}
\newcommand{\bfomega}{\boldsymbol{\omega}}

\newcommand{\bfE}{\mathbf{E}}

\newcommand{\bfR}{\mathbf{R}}

\newcommand{\bfX}{\mathbf{X}}

\newcommand{\bbR}{\mathbb{R}}

\title{\LARGE \bf Active Bayesian Multi-class Mapping from Range and Semantic Segmentation Observations%
\thanks{We gratefully acknowledge support from ARL DCIST CRA W911NF-17-2-0181 and ONR N00014-18-1-2828. The Unity simulation used for evaluation is developed by ARL for the DCIST project: \url{www.dcist.org}.}%
}

\author{Arash Asgharivaskasi \and Nikolay Atanasov
\thanks{The authors are with the Department of Electrical and Computer Engineering, University of California San Diego, CA 92093, USA {\tt\small \{aasghari,natanasov\}@eng.ucsd.edu}.}
}

\begin{document}

\maketitle

\begin{abstract}
Many robot applications call for autonomous exploration and mapping of unknown and unstructured environments. Information-based exploration techniques, such as Cauchy-Schwarz quadratic mutual information (CSQMI) and fast Shannon mutual information (FSMI), have successfully achieved active binary occupancy mapping with range measurements. However, as we envision robots performing complex tasks specified with semantically meaningful objects, it is necessary to capture semantic categories in the measurements, map representation, and exploration objective. This work develops a Bayesian multi-class mapping algorithm utilizing range-category measurements. We derive a closed-form  efficiently computable lower bound for the Shannon mutual information between the multi-class map and the measurements. The bound allows rapid evaluation of many potential robot trajectories for autonomous exploration and mapping. We compare our method against frontier-based and FSMI exploration and apply it in a 3-D photo-realistic simulation environment.
\end{abstract}




\section{Introduction}
\label{sec:introduction}

Real-time understanding, accurate modeling, and efficient storage of a robot's environment are key capabilities for autonomous operation. Occupancy grid mapping \cite{occ_mapping_1, occ_mapping_2} is a simple, yet widely used and effective, technique for distinguishing between traversable and occupied space surrounding a mobile robot. However, as our vision of delegating increasingly sophisticated tasks to autonomous robots expands, so should the representation power of online mapping algorithms. Augmenting traditional geometric models with semantic information about the context and object-level structure of the environment has become a mainstream problem in robotics \cite{semantic_1,semantic_2,semantic_3}. Robots are also increasingly expected to operate in unknown environments, with little to no prior information, in applications such as disaster response, environmental monitoring, and reconnaissance. This calls for algorithms allowing robots to autonomously explore unknown environments and construct low-uncertainty metric-semantic maps in real-time, while taking collision and visibility constraints into account.

This paper considers the active metric-semantic mapping problem, requiring a robot to explore and map an unknown environment, relying on streaming distance and object category observations, e.g., generated by semantic segmentation over RGBD images \cite{bonnet}. Our approach extends information-theoretic active mapping techniques \cite{julian, csqmi, fsmi} from binary to multi-class environment representations. We introduce a Bayesian multi-class mapping procedure which maintains a probability distribution over semantic categories and updates it via a probabilistic range-category perception model. Our main \textbf{contribution} is the derivation of a closed-form efficiently computable lower bound for the Shannon mutual information between a multi-class occupancy map and a set of range-category measurements. This lower bound allows rapid evaluation of many potential robot trajectories online to (re-)select one that leads to the best trade-off between uncertainty reduction and efficient exploration of the metric-semantic map. Unlike traditional class-agnostic exploration methods, our model and information measure incorporate the uncertainty of different semantic classes, leading to faster and more accurate exploration. The proposed approach relies on general range and class measurements and general motion kinematics, making it suitable for either ground or aerial robots, equipped with either camera or Lidar sensors, exploring either indoor or outdoor environments.


\begin{figure}[t]
  \centering
  \includegraphics[width=\linewidth]{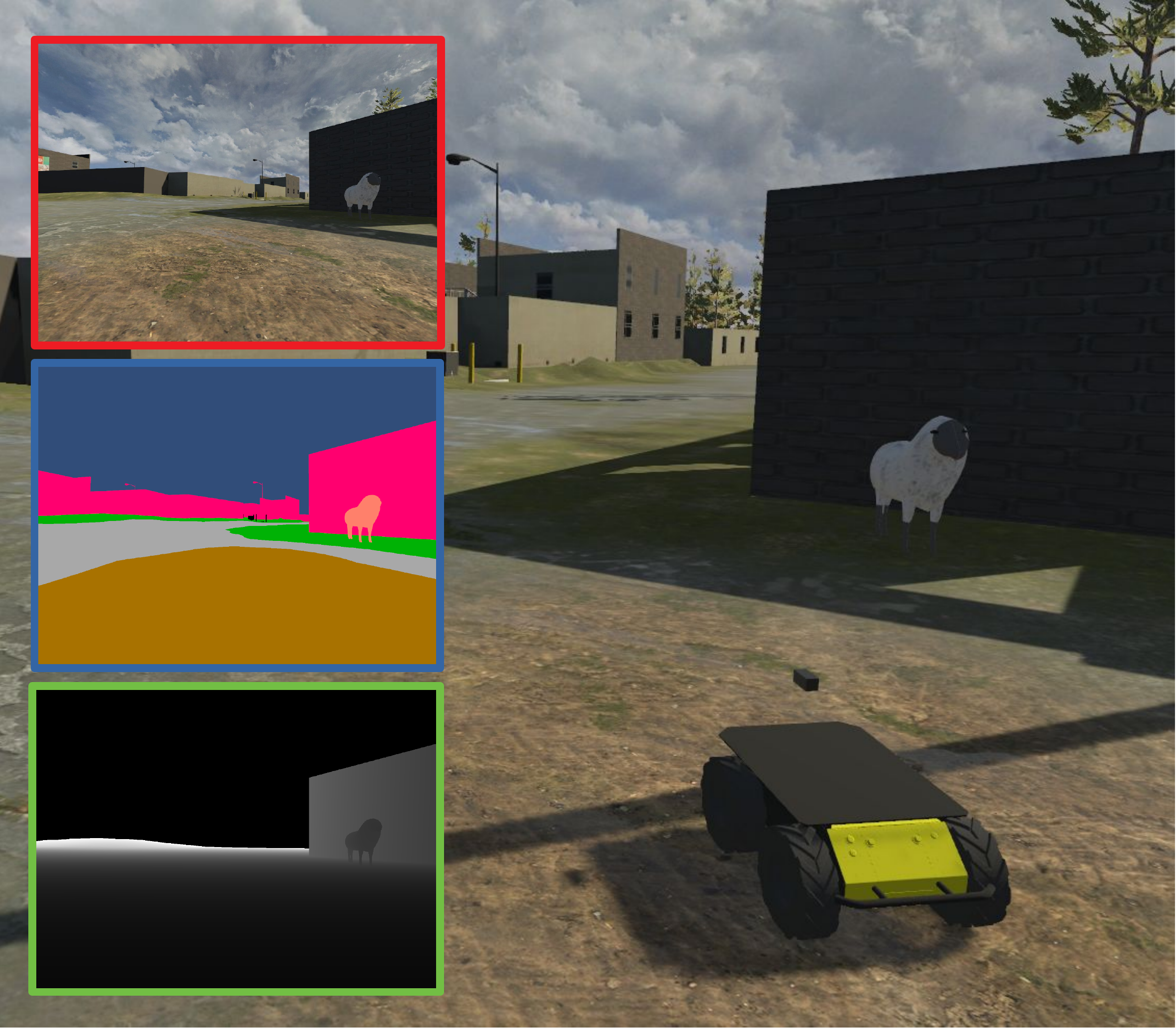}
  \caption{A robot autonomously explores an unknown environment using an RGBD sensor and a semantic segmentation algorithm.}
    \label{fig:prob_state}
\end{figure}

\section{Related Work}
\label{sec:related_work}

Frontier-based exploration \cite{frontier} is a seminal work that highlights the utility of autonomous exploration and active mapping. It inspired methods \cite{geo_exp_1, geo_exp_2} that rely on geometric features, such as the boundaries between free and unknown space (frontiers) and the volume that would be revealed by new sensor observations. Due to their intuitive formulation and low computational requirements, geometry-based methods continue to be widely employed in active perception. Recent works include semantics-assisted indoor exploration \cite{geo_exp_3}, active 3D coverage and reconstruction \cite{geo_exp_4}, and Laplace potential fields for safe outdoor exploration \cite{geo_exp_5}.




Alternative techniques for active mapping use probabilistic environment models and information-theoretic utility functions to measure and minimize the model uncertainty. The work by Elfes \cite{info_exp_1} is among the first to propose an information-based utility function. Information-based exploration strategies have been devised for uncertainty minimization in robot localization or environment mapping \cite{info_exp_2, info_exp_3, info_exp_4}. Information-theoretic objectives, however, require integration over the potential sensor measurements, limiting the use of direct numerical approximations to short planning horizons. Kollar and Roy \cite{info_exp_5} formulated active mapping using an extended Kalman filter and proposed a local-global optimization, leading to significant gains in efficiency for uncertainty computation and long-horizon planning. Unlike geometry-based approaches, information-theoretic exploration can be directly formulated for active simultaneous localization and mapping (SLAM) \cite{active_slam_1,active_slam_2,active_slam_3,active_slam_4}, aiming to determine a sensing trajectory that balances robot state uncertainty and visitation of unexplored map regions. Stachniss et al. \cite{info_exp_7} approximate information gain for a Rao-blackwellized particle filter over the joint state of robot pose and map occupancy. Julian et al. \cite{julian} prove that, for range measurements, the Shannon mutual information is maximized over trajectories that visit unexplored areas. However, without imposing further structure over the observation model, computing the mutual information objective requires numerical integration. The need for efficient mutual information computation becomes evident in 3-D environments. Cauchy-Schwarz quadratic mutual information (CSQMI) \cite{csqmi} and fast Shannon mutual information (FSMI) \cite{fsmi} offer efficiently computable closed-form objectives for active occupancy mapping with range measurements. Henderson et al. \cite{info_exp_11} propose an even faster computation based on a recursive expression for Shannon mutual information in continuous maps.



This paper is most related to CSQMI~\cite{csqmi} and FSMI~\cite{fsmi} in that it develops a closed-form expression for mutual information. However, instead of a binary map and range-only measurements, our formulation considers a multi-class map with Bayesian updates from range-category measurements.



\section{Problem Statement}
\label{sec:problem_statement}

Consider a robot with pose $\bfX_t \in SE(3)$ at time $t$ and deterministic discrete-time kinematics:
\begin{equation}
    \bfX_t := \begin{bmatrix} \bfR_t & \bfp_t \\ \mathbf{0}^\top & 1 \end{bmatrix}, \qquad \bfX_{t+1} =  \bfX_{t}\exp{(\hat{\bfu}_t)},
\label{eq:dynamic_model}
\end{equation}
where $\bfR_t \in SO(3)$ is the robot orientation, $\bfp_t \in \mathbb{R}^3$ is the robot position, and $\bfu_t := [\bfv_t^\top, \bfomega_t^\top]^\top \in \calU \subset \mathbb{R}^6$ is the control input, consisting of linear velocity $\bfv_t \in \mathbb{R}^3$ and angular velocity $\bfomega_t \in \mathbb{R}^3$. The function $\hat{(\cdot)}: \bbR^6 \rightarrow \mathfrak{se}(3)$ maps vectors in $\mathbb{R}^6$ to the Lie algebra $\mathfrak{se}(3)$. The robot is navigating in an environment represented as a collection of disjoint sets $\calE_k \subset \mathbb{R}^3$, each associated with a semantic category $k \in \calK := \crl{0,1,\ldots,K}$. Let $\calE_0$ denote free space, while each $\calE_k$ for $k >0$ represents a different category, such as buildings, animals, terrain (see Fig.~\ref{fig:prob_state}).



We assume that the robot is equipped with a sensor that provides information about the distance to and semantic categories of surrounding objects along a set of rays $\calF := \crl{\bfeta_b}_b$, where $\bfeta_b \in \mathbb{R}^3$ with $\|\bfeta_b\|_2 = r_{max}$ and $r_{max} > 0$ is the maximum sensing range.

\begin{definition}
A \emph{sensor observation} at time $t$ from robot pose $\bfX_t$ is a collection $\calZ_t := \crl{\bfz_{t,b}}_b$ of range and category measurements $\bfz_{t,b}:= (r_{t,b}, y_{t,b}) \in \mathbb{R}_{\geq 0} \times \calK$, acquired along the sensor rays $\bfR_t\bfeta_b$ with $\bfeta_b \in \calF$ at robot position $\bfp_t$. 
\end{definition}





Such information may be obtained by processing the observations of an RGBD camera or a Lidar with a semantic segmentation algorithm~\cite{bonnet}. The goal is to construct a multi-class map $\bfm$ of the environment based on the labeled range measurements. We model $\bfm$ as a grid of cells $m_i$, $i \in \calI := \{1, \ldots, N\}$, each labeled with a category $m_i \in \calK$. In order to model noisy sensor observations, we consider a probability density function (PDF) $p(\calZ_t \mid \bfm, \bfX_t)$. This observation model allows integrating the measurements into a probabilistic map representation using Bayesian updates. Let $p_t(\bfm) := p(\bfm \mid \calZ_{1:t}, \bfX_{1:t})$ be the PDF of the map $\bfm$ given the robot trajectory $\bfX_{1:t}$ and observations $\calZ_{1:t}$ up to time $t$. Given a new observation $\calZ_{t+1}$ obtained from robot pose $\bfX_{t+1}$, the Bayesian update to the map PDF is:
\begin{equation}
\label{eq:bayes_rule}
p_{t+1}(\bfm) \propto p(\calZ_{t+1} | \bfm, \bfX_{t+1}) p_t(\bfm).
\end{equation}
We assume that the robot pose is known and omit the dependence of the map distribution and the observation model on it for brevity. We consider the following problem.



%
\begin{problem}
Given a prior map PDF $p_t(\bfm)$ and a finite planning horizon $T$, find a control sequence $\bfu_{t:t+T-1}$ for the model in \eqref{eq:dynamic_model} that maximizes the ratio of mutual information $I\prl{\bfm; \calZ_{t+1:t+T} \mid \calZ_{1:t}}$ between the map $\bfm$ and future sensor observations $\calZ_{t+1:t+T}$ and the motion cost $J(\bfX_{t:t+T-1},\bfu_{t:t+T-1})$ of the planned robot trajectory:
\begin{equation}
\label{eq:sem_exp}
\begin{aligned}
\max_{\bfu_{t:t+T-1}} \frac{I\prl{\bfm; \calZ_{t+1:t+T} \mid \calZ_{1:t}}}{J(\bfX_{t:t+T-1}, \bfu_{t:t+T-1})} 
 \;\;\text{subject to}\;\;  \eqref{eq:dynamic_model}, \eqref{eq:bayes_rule}.
\end{aligned}
\end{equation}
\end{problem}


The precise definitions of the mutual information and motion cost terms above are:
\begin{align}
&I\prl{\bfm; \calZ_{t+1:t+T} | \calZ_{1:t}} := \!\sum_{\bfm \in \calK^N} \!\int \cdots \int p(\bfm, \calZ_{t+1:t+T} | \calZ_{1:t}) \notag\\
& \quad \times \log \frac{p(\bfm, \calZ_{t+1:t+T} \mid \calZ_{1:t})}{p(\bfm | \calZ_{1:t})p(\calZ_{t+1:t+T}|\calZ_{1:t})} \prod_{\tau = 1}^T \prod_b d\bfz_{t+\tau,b} \label{eq:mutual-information}\\
&J(\bfX_{t:t+T-1}, \bfu_{t:t+T-1}) := q(\bfX_{t+T}) + \sum_{\tau = 0}^{T-1} c(\bfX_{t+\tau},\bfu_{t+\tau}), \notag
\end{align}
where $q(\bfX), c(\bfX,\bfu) > 0$ model terminal and stage motion costs (e.g., distance traveled, elapsed time), respectively.

We develop a multi-class extension to the log-odds occupancy mapping algorithm \cite[Ch.~9]{ProbabilisticRoboticsBook} in Sec.~\ref{sec:bayes_multi_class_mapping} and derive an efficient approximation to the mutual information term in Sec.~\ref{sec:info_plan}. This allows us to evaluate potential robot trajectories online and (re-)select the one that maximizes the objective in~\eqref{eq:sem_exp}, leading to efficient active multi-class mapping.


\section{Bayesian Multi-class Mapping}
\label{sec:bayes_multi_class_mapping}


This section derives the Bayesian update in \eqref{eq:bayes_rule}, using a \emph{multinomial logit model} to represent $p_t(\bfm)$. To ensure that the number of parameters in the model scales linearly with the map size $N$, we maintain a factorized PDF over the cells:
\begin{equation}
\label{eq:pdf_factorization}
p_t(\bfm) = \prod_{i=1}^N p_t(m_i).
\end{equation}
We represent the individual cell PDFs $p_t(m_i)$ over $\calK$ using a vector of log odds:
\begin{equation}
\bfh_{t,i} := \begin{bmatrix} \log \frac{p_t(m_i = 0)}{p_t(m_i = 0)} & \cdots & \log \frac{p_t(m_i = K)}{p_t(m_i = 0)} \end{bmatrix}^\top \!\!\in \mathbb{R}^{K+1},
\end{equation}
where the free-class likelihood $p_t(m_i = 0)$ is used as a pivot. Given the log-odds vector $\bfh_{t,i}$, the PDF of cell $m_i$ may be recovered using the softmax function $\sigma:\mathbb{R}^{K+1} \mapsto \mathbb{R}^{K+1}$:
\begin{equation}
p_t(m_i = k) = \sigma_{k+1}(\bfh_{t,i}) := \frac{\bfe_{k+1}^\top \exp(\bfh_{t,i})}{ \mathbf{1}^\top \exp(\bfh_{t,i})},
\end{equation}
where $\bfe_k$ is the standard basis vector with $k$-th element equal to $1$ and $0$ elsewhere, and $\mathbf{1}$ is the vector with all elements equal to $1$. To derive Bayes rule for the log-odds $\bfh_{t,i}$, we need to specify an observation model for the measurements.


\begin{definition}
The \emph{inverse observation model} of a range-category measurement $\bfz$ obtained from robot pose $\bfX$ along sensor ray $\bfeta \in \calF$ with respect to map cell $m_i$ is a probability density function $p(m_i | \bfz; \bfX, \bfeta)$.
\end{definition}

The Bayesian update in \eqref{eq:bayes_rule} for $\bfh_{t,i}$ can be obtained in terms of the range-category inverse observation model, evaluated at the new measurement set $\calZ_{t+1}$.

\begin{proposition}
\label{prop:log-odds-bayes-rule}
Let $\bfh_{t,i}$ be the log odds of cell $m_i$ at time $t$. Given sensor observation $\calZ_{t+1}$, the posterior log-odds are:
\begin{equation}
\label{eq:log-odds-bayes-rule}
\bfh_{t+1,i} = \bfh_{t,i} + \sum_{\bfz \in \calZ_{t+1}} \prl{ \bfl_i(\bfz) - \bfh_{0,i}}
\end{equation}
where $\bfl_i(\bfz)$ is the inverse observation model log odds:
\begin{equation}
\bfl_i(\bfz) := \begin{bmatrix} \log \frac{p(m_i = 0 | \bfz)}{p(m_i = 0 | \bfz)} & \cdots & \log \frac{p(m_i = K | \bfz)}{p(m_i = 0 | \bfz)} \end{bmatrix}^\top.
\end{equation}
\end{proposition}

\begin{proof}
See Appendix~\ref{app:log-odds-bayes-rule}.
\end{proof}


To complete the Bayesian multi-class mapping algorithm suggested by \eqref{eq:log-odds-bayes-rule} we need a particular inverse observation model. When a sensor measurement is generated, the sensor ray continues to travel until it hits an obstacle of category $\calK \setminus \{0\}$ or reaches the maximum sensing range $r_{max}$. The resulting labeled range measurement $\bfz = (r,y)$ indicates that map cell $m_i$ is occupied if the measurement end point $\bfp + \frac{r}{r_{max}} \bfR \bfeta$ lies in the cell. If $m_i$ lies along the sensor ray but does not contain the end point, it is observed as free. Finally, if $m_i$ is not intersected by the sensor ray, no information is provided about its occupancy. The map cells along the sensor ray can be determined by a rasterization algorithm, such as Bresenham's line algorithm \cite{bresenham}. We parameterize the inverse observation model log-odds vector as:
\begin{gather}
\label{eq:log_inverse_observation_model}
\bfl_i((r,y)) := \begin{cases}
\bfphi^+ + \bfE_{y+1}\bfpsi^+, & \text{$r$ indicates $m_i$ is occupied},\\
\bfphi^-, & \text{$r$ indicates $m_i$ is free},\\
\bfh_{0,i}, & \text{otherwise},
\end{cases}
\raisetag{3ex}
\end{gather}
where $\bfE_k := \bfe_k \bfe_k^\top$ and $\bfpsi^+\!, \bfphi^-\!, \bfphi^+ \in \mathbb{R}^{K+1}$ are parameter vectors, whose first element is $0$ to ensure that $\bfl_i(\bfz)$ is a valid log-odds vector. This parameterization leads to inverse observation model $p(m_i = k | \bfz) = \sigma_{k+1}(\bfl_i(\bfz))$, which is piece-wise constant along the sensor ray.


To compute the mutual information between an observation sequence $\calZ_{t+1:t+T}$ and the map $\bfm$ in the next section, we will also need the PDF of a range-category measurement $\bfz_{\tau,b} \in \calZ_{t+1:t+T}$ conditioned on $\calZ_{1:t}$. Let $\calR_{\tau,b}(r) \subset \calI$ denote the set of map cell indices along the ray $\bfR_\tau \bfeta_b$ from the robot position $\bfp_\tau$ with length $r$. Let $\gamma_{\tau,b}(i)$ denote the distance traveled by the ray $\bfR_\tau \bfeta_b$ within cell $m_i$ and $i_{\tau,b}^* \in \calR_{\tau,b}(r)$ denote the index of the cell hit by $\bfz_{\tau,b}$. We define the PDF of $\bfz_{\tau,b} = (r, y)$ conditioned on $\calZ_{1:t}$ as:
\begin{gather}
\label{eq:cond_prob_approx}
p(\bfz_{\tau,b} | \calZ_{1:t}) = \frac{p_t(m_{i_{\tau,b}^*} = y)}{\gamma_{\tau,b}(i_{\tau,b}^*)} \negquad \prod_{i \in \calR_{\tau,b}(r) \setminus \{i_{\tau,b}^*\}} \negquad p_t(m_i = 0).
\raisetag{2ex}
\end{gather}
This definition states that the likelihood of $\bfz_{\tau,b}= (r,y)$ at time $t$ depends on the likelihood that the cells $m_i$ along the ray $\bfR_{\tau}\bfeta_b$ of length $r$ are empty and the likelihood that the hit cell $m_{i_{\tau,b}^*}$ has class $y$. A similar model for binary observations has been used in \cite{julian,csqmi,fsmi}. Knowing how an observation affects the map PDF $p_t(\bfm)$, we now switch our focus to computing of the mutual information between a sequence of observations $\calZ_{t+1:t+T}$ and the multi-class occupancy map $\bfm$.

\section{Informative Planning}
\label{sec:info_plan}

Computing the mutual information term in \eqref{eq:mutual-information} is challenging because it involves integration over all possible values of the observation sequence $\calZ_{t+1:t+T}$. Our main result is an efficiently computable lower bound on $I\prl{\bfm; \calZ_{t+1:t+T} | \calZ_{1:t}}$ for range-category observations $\calZ_{t+1:t+T}$ and a multi-class occupancy map $\bfm$. The result is obtained by selecting a subset $\underline{\calZ}_{t+1:t+T} = \crl{\bfz_{\tau,b}}_{\tau=t+1,b=1}^{t+T,B}$ of the observations $\calZ_{t+1:t+T}$ in which the sensor rays are non-overlapping. Precisely, any pair of measurements $\bfz_{\tau,b}$, $\bfz_{\tau',b'} \in \underline{\calZ}_{t+1:t+T}$ satisfies:
\begin{equation}
\label{eq:nonoverlapping}
\calR_{\tau,b}(r_{max}) \cap \calR_{\tau',b'}(r_{max}) = \emptyset.
\end{equation}
In practice, constructing $\underline{\calZ}_{t+1:t+T}$ requires removing intersecting rays from $\calZ_{t+1:t+T}$ to ensure that the remaining observations are mutually independent. Consequently, the mutual information between $\bfm$ and $\underline{\calZ}_{t+1:t+T}$ can be obtained as a sum of mutual information terms between single rays $\bfz_{\tau,b} \in \underline{\calZ}_{t+1:t+T}$ and map cells $m_i$ observed by $\bfz_{\tau,b}$. This idea is inspired by CSQMI~\cite{csqmi} but we generalize their results to multi-class observations and map.

\begin{proposition}
\label{prop:mut_inf_semantic}
Given a sequence of labeled range observations $\calZ_{t+1:t+T}$, let $\underline{\calZ}_{t+1:t+T} = \crl{\bfz_{\tau,b}}_{\tau=t+1,b=1}^{t+T,B}$ be a subset of non-overlapping measurements that satisfy \eqref{eq:nonoverlapping}. Then, the Shannon mutual information between $\calZ_{t+1:t+T}$ and a multi-class occupancy map $\bfm$ can be lower bounded as:
\begin{equation}
\label{eq:mut_inf_semantic}
\begin{aligned}
I(\bfm; &\calZ_{t+1:t+T} | \calZ_{1:t}) \geq I\prl{\bfm; \underline{\calZ}_{t+1:t+T} | \calZ_{1:t}}\\
&=\sum_{\tau=t+1}^{t+T} \sum_{b=1}^{B} \sum_{k=1}^K \sum_{n=1}^{N_{\tau,b}} p_{\tau,b}(n, k) C_{\tau,b}(n, k),
\end{aligned}
\end{equation}
where $N_{\tau,b} := | \calR_{\tau,b}(r_{max}) |$,
\begin{equation*}
\begin{aligned}
p_{\tau,b}(n, k) &:= p_t(m_{i_{\tau,b}^*} = k) \prod_{i \in \Tilde{\calR}_{\tau,b}(n) \setminus \{i_{\tau,b}^*\}} p_t(m_i = 0),\\
C_{\tau,b}(n, k) &:= f(\bfphi^+ + \bfE_{k+1}\bfpsi^+ - \bfh_{0,i_{\tau,b}^*}, \bfh_{t,i_{\tau,b}^*})\\
& + \sum_{i \in \Tilde{\calR}_{\tau,b}(n) \setminus \{i_{\tau,b}^*\}} f(\bfphi^--\bfh_{0,i}, \bfh_{t,i}),\\
f(\bfphi, \bfh) &:= \log\prl{ \frac{\mathbf{1}^\top \exp(\bfh)}{\mathbf{1}^\top \exp(\bfphi + \bfh)} } + \bfphi^\top \sigma(\bfphi + \bfh),
\end{aligned}
\end{equation*}
and $\Tilde{\calR}_{\tau,b}(n) \subseteq \calR_{\tau,b}(r_{max})$ is the set of the first $n$ map cell indices along the ray $\bfR_{\tau}\bfeta_b$, i.e., $\Tilde{\calR}_{\tau,b}(n) := \{i \mid i \in \calR_{\tau,b}(r), |\calR_{\tau,b}(r)| = n, r \leq r_{max}\}$.
\end{proposition}

\begin{proof}
See Appendix~\ref{app:mut-inf-semantic}.
\end{proof}

Prop.~\ref{prop:mut_inf_semantic} allows evaluating the informativeness according to \eqref{eq:sem_exp} of any potential robot trajectory $\bfX_{t:t+T}$, $\bfu_{t:t+T-1}$. We use a motion planning algorithm to obtain a set of trajectories to the map frontiers, determined from the current map PDF $p_t(\bfm)$. Alg.~\ref{alg:sem_exp} summarizes the procedure for determining a state-control trajectory $\bfX^*_{t:t+T}$, $\bfu^*_{t:t+T-1}$ that maximizes the objective in \eqref{eq:sem_exp}. This kinematically feasible trajectory can be tracked by a low-level controller that takes the robot dynamics into account. We evaluate the proposed active multi-class mapping algorithm next.

\begin{algorithm}[t]
\caption{Information-theoretic Path Planning}\label{alg:sem_exp}
\begin{algorithmic}[1]
  \renewcommand{\algorithmicrequire}{\textbf{Input:}}
  \renewcommand{\algorithmicensure}{\textbf{Output:}}
  \Require $\bfX_t$, $p_t(\bfm)$
  \State $\calF = \Call{findFrontiers}{p_t(\bfm)}$
  \For{$f \in \calF$}
    \State $\bfX_{t+1:t+T}, \bfu_{t:t+T-1} = \Call{planPath}{\bfX_t,p_t(\bfm),f}$
    \State Compute \eqref{eq:sem_exp} over $\bfX_{t:t+T}, \bfu_{t:t+T-1}$ via Prop.~\eqref{prop:mut_inf_semantic}
  \EndFor
  \State \Return $\bfX^*_{t:t+T}$, $\bfu^*_{t:t+T-1}$ with highest value
\end{algorithmic}
\end{algorithm}


\section{Experiments}
\label{sec:experiments}

We compare our active multi-class mapping algorithm against two baseline exploration strategies: frontier-based exploration (Frontier) \cite{frontier} and FSMI~\cite{fsmi}. We compare the methods in an active binary mapping scenario in Sec.~\ref{subsec:exp_bin} and active multi-class mapping scenario in Sec.~\ref{subsec:exp_homogen}. All three methods use our range-category sensor model in \eqref{eq:log_inverse_observation_model}
and our Bayesian multi-class mapping in \eqref{eq:log-odds-bayes-rule} but select informative robot trajectories $\bfX_{t:t+T}$, $\bfu_{t:t+T-1}$ based on their own criteria. Finally, in Sec.~\ref{subsec:exp_3d}, we apply our method in a 3-D simulated Unity environment to demonstrate large-scale realistic active multi-class mapping.

To identify frontiers we apply edge detection on the most likely map at time $t$ (the mode of $p_t(\bfm)$). Then, we cluster the edge cells by detecting the connected components of the boundaries between explore and unexplored space. We plan a path from the robot pose $\bfX_t$ to the center of each frontier using $A^*$ graph search. To ensure feasibility of the planned paths, we fit a piece-wise polynomial trajectory to each path and provide it to a low-level controller to generate $\bfu_{t:t+T-1}$.

\begin{figure}[t]
  \includegraphics[width=0.48\linewidth]{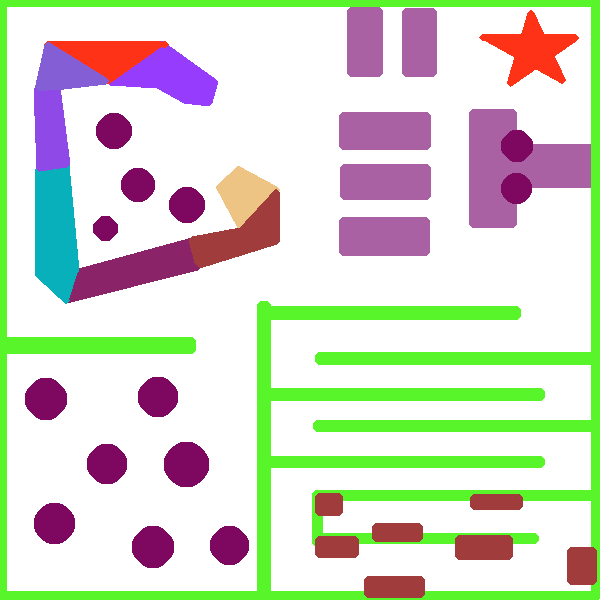}%
  \hfill%
  \includegraphics[width=0.48\linewidth]{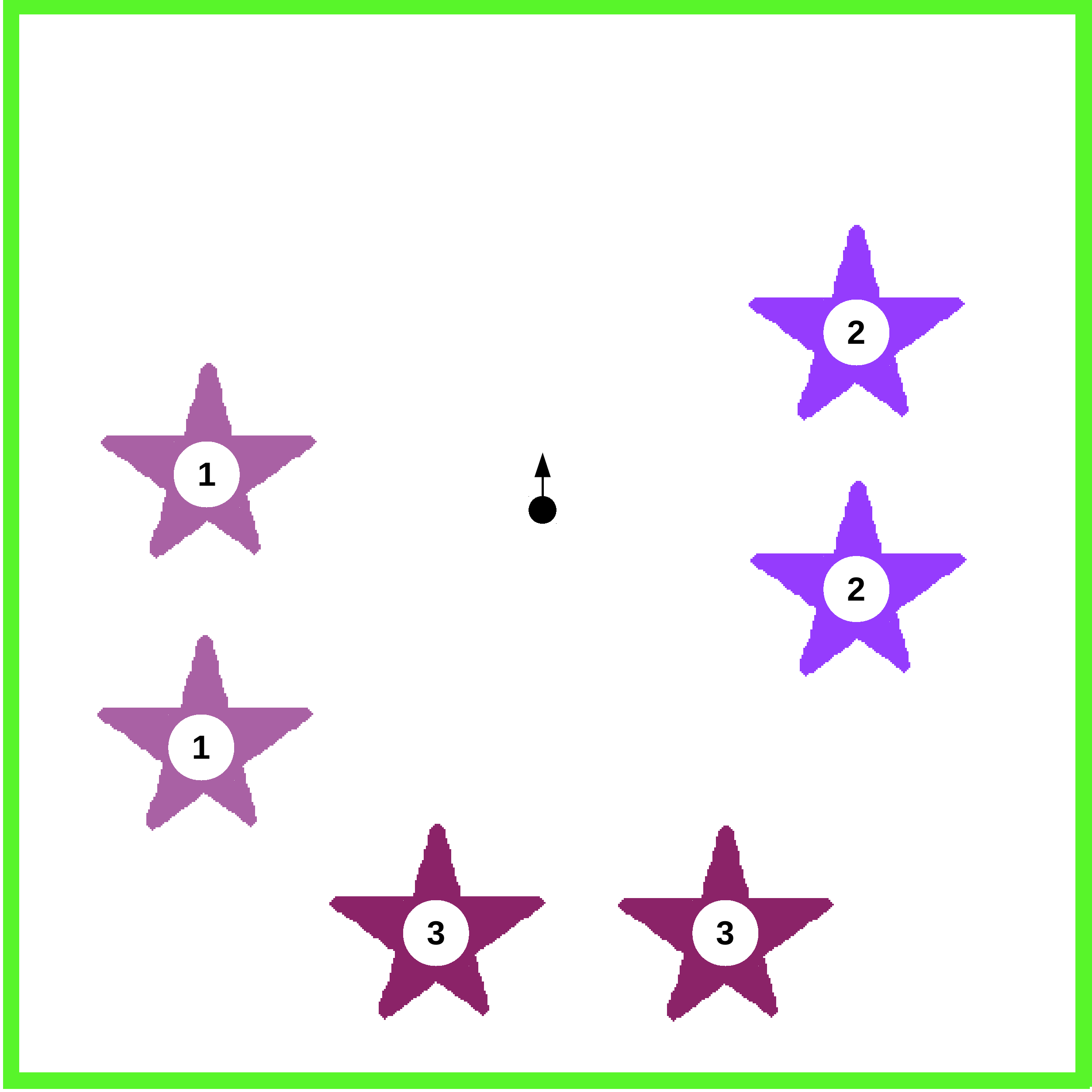}
  \caption{Synthetic environments used for comparisons among frontier-based exploration \cite{frontier}, FSMI \cite{fsmi}, and our approach. Different semantic categories are represented by distinct colors.}
  \label{fig:homogen_true_map}
\end{figure}

\begin{figure}[t]
  \centering
  \includegraphics[width=\linewidth]{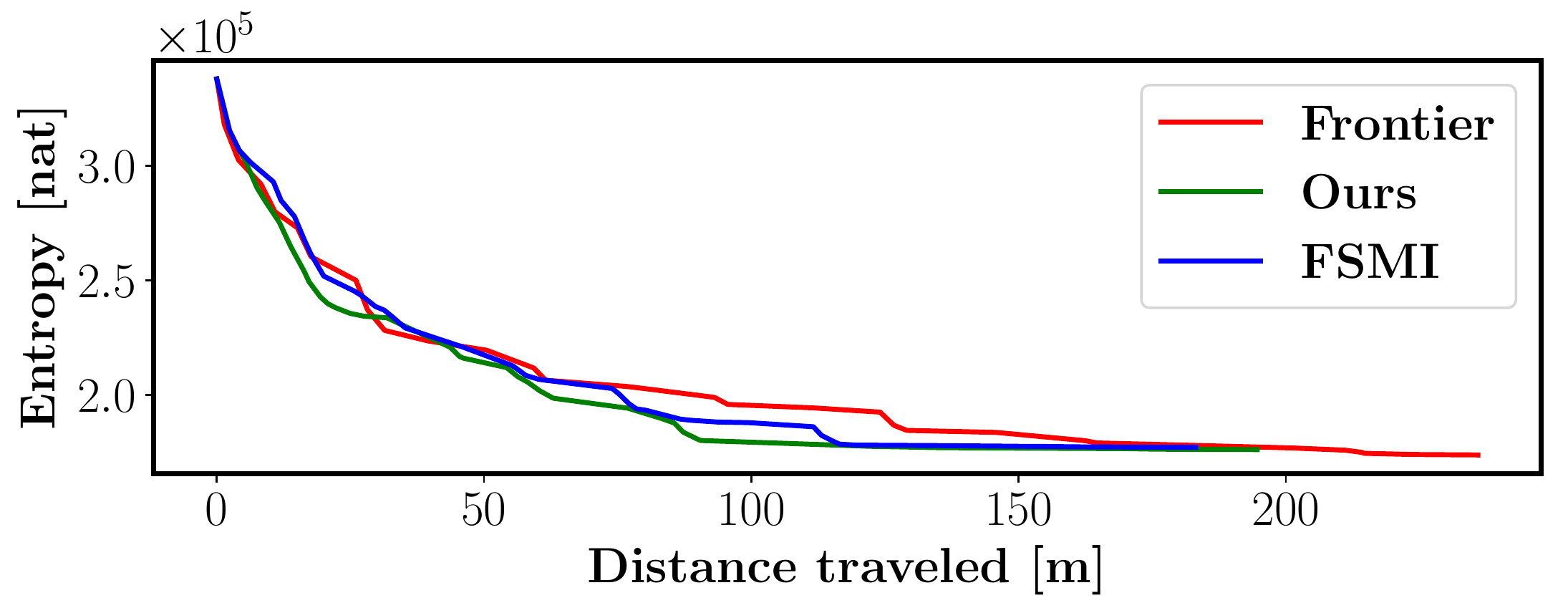}\\
  \includegraphics[width=\linewidth]{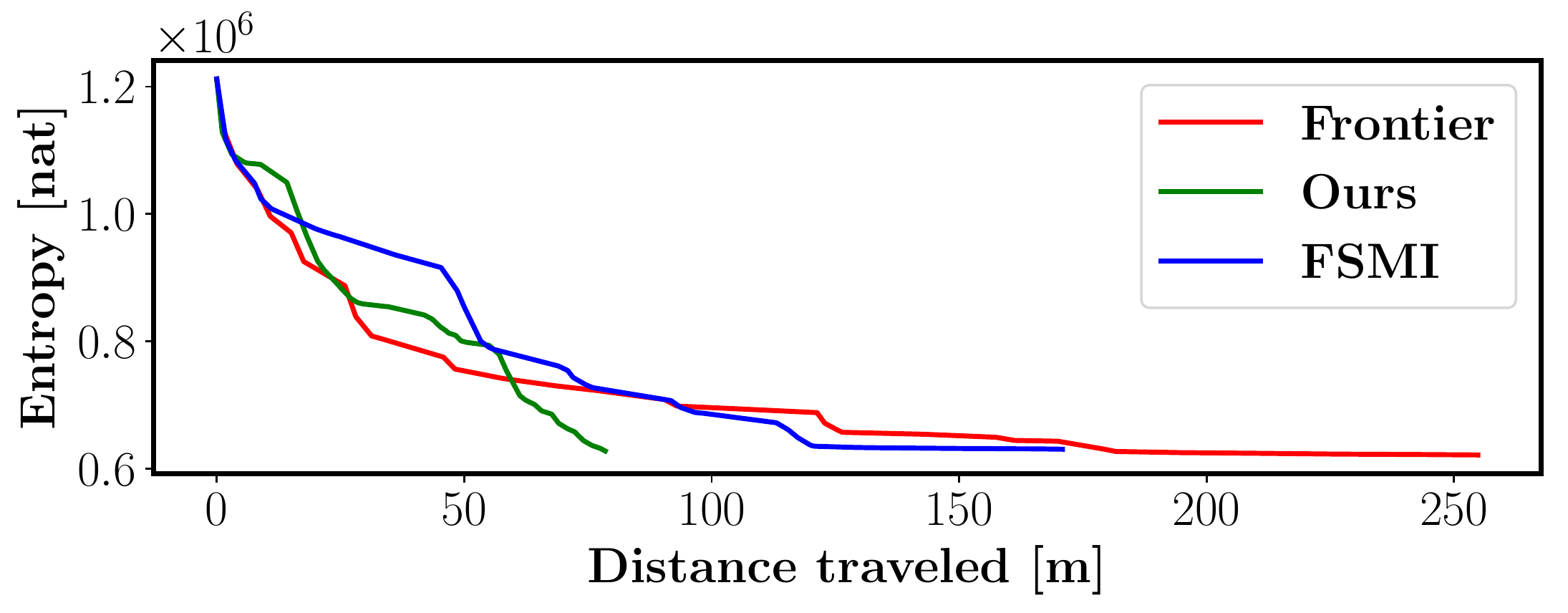}
  \caption{Simulation results for active binary occupancy mapping (top) and active multi-class occupancy mapping (bottom) on the environments in Fig.~\ref{fig:homogen_true_map}. In both cases, the results are averaged over $5$ random experiments with random initial robot pose.}
\label{fig:2d_results}
\end{figure}




\subsection{2-D Binary Exploration}
\label{subsec:exp_bin}

We consider active binary occupancy mapping first. We compare our method against Frontier and FSMI in synthetic 2-D environments (Fig.~\ref{fig:homogen_true_map}). A 2-D LiDAR sensor is simulated with additive Gaussian noise $\calN(0,0.03)$. Fig.~\ref{fig:2d_results} (top) compares the map entropy reduction per distance traveled among the three methods. Our method performs similarly to FSMI in that both achieve low map entropy by traversing significantly less distance compared to Frontier.

\subsection{2-D Multi-class Exploration}
\label{subsec:exp_homogen}

Next, we use the same 2-D environments in Fig.~\ref{fig:homogen_true_map} but introduce range-category measurements. Range measurements are subject to additive Gaussian noise $\calN(0,0.03)$, while category measurement have a uniform misclassification probability of $0.2$. Fig.~\ref{fig:2d_results} (bottom) compares the multi-class entropy reduction per distance travelled for all three strategies. Our method reaches the same level of map entropy as FSMI and Frontier but traverses a noticeable shorter distance. This can be attributed to the fact that only our method distinguishes map cells whose occupancy probabilities are the same but their per-class probabilities differ from each other.

\subsection{3-D Multi-class Exploration}
\label{subsec:exp_3d}

Finally, we evaluate our active multi-class mapping algorithm in a photo-realistic 3-D Unity simulation, shown in Fig.~\ref{fig:3d_sim_env} (f). We use a Husky robot equipped with an RGBD camera and run a semantic segmentation algorithm over the RGB images. The range measurements have an additive Gaussian noise of $\calN(0,0.1)$. The semantic segmentation algorithm detects the true class with a probability of $0.95$. We implemented an octree version of our Bayesian multi-class mapping algorithm and information computation to scale the method to a large 3-D environment. Our implementation extends the OctoMap algorithm \cite{octomap} to multi-class log-odds mapping as described in Sec.~\ref{sec:bayes_multi_class_mapping}. The octree map is compressed by merging cells whose multi-class probabilities reach the same max or min threshold. This allows compressing the majority of the free cells into few large cells. The details of the octree implementation of the active multi-class mapping algorithm will be presented in a follow-up paper. Fig.~\ref{fig:3d_sim_env} (a)-(e) shows several iterations of the exploration process. Fig.~\ref{fig:3d_sim_res} shows the change in map entropy versus elapsed time and the classification precision for every observed semantic class.


\begin{figure}[t]
  \centering
  \includegraphics[width=\linewidth]{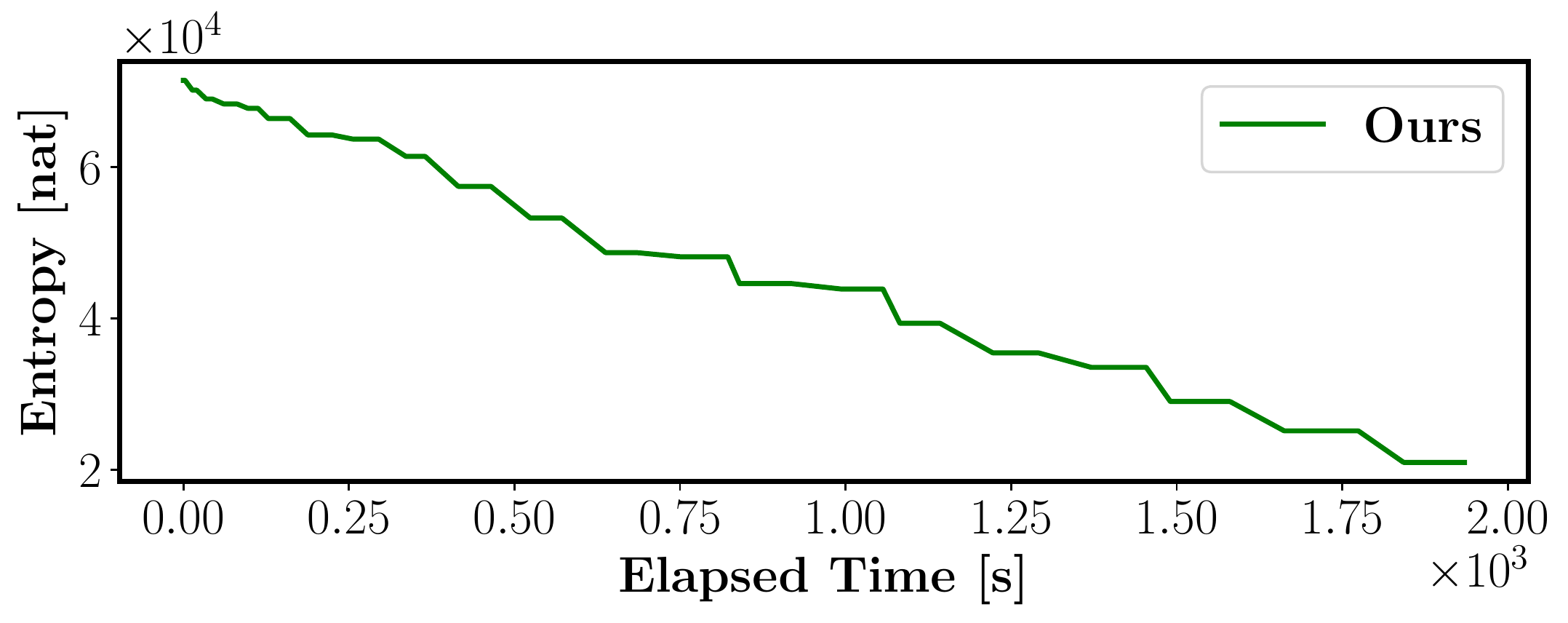}\\
  \includegraphics[width=\linewidth,trim=0mm 0mm 0mm 10mm, clip]{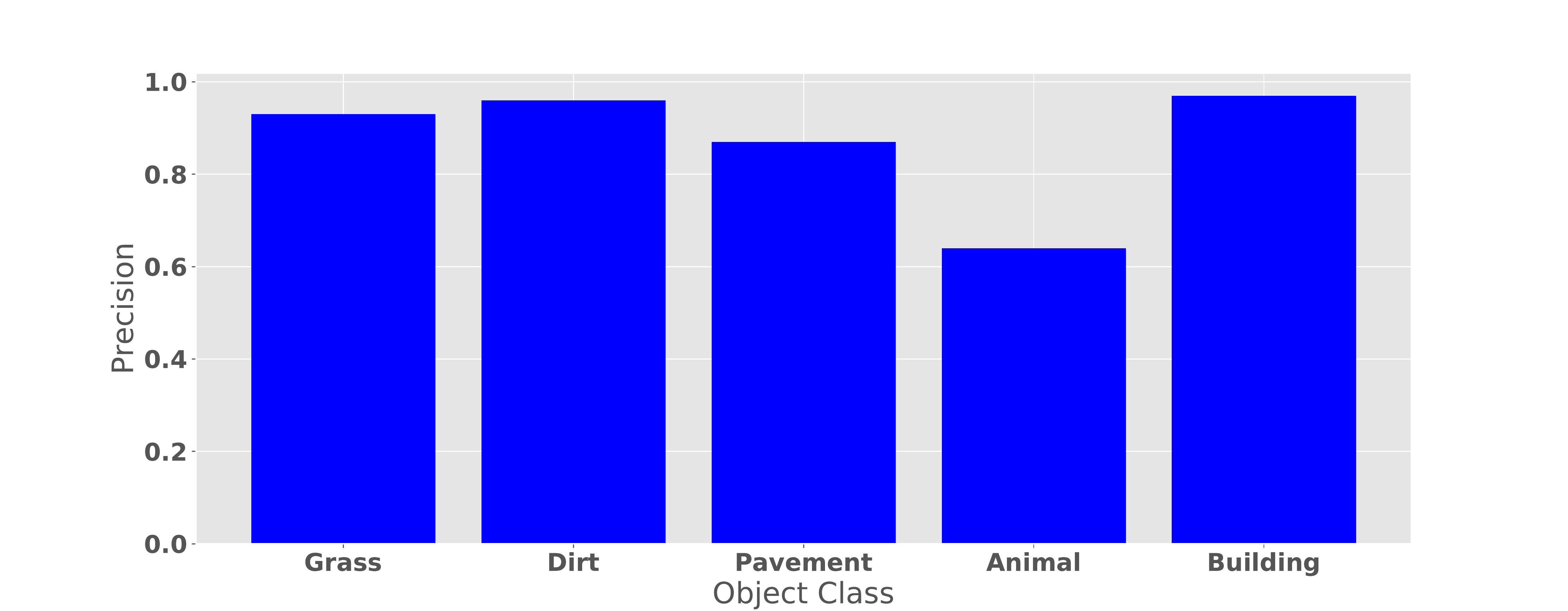}
  \caption{(top): Entropy versus time for 3-D simulations. The flat line segments represent the planning phase. (bottom): Precision for observed semantic classes.}
\label{fig:3d_sim_res}
\end{figure}

\begin{figure*}[t]
    \begin{subfigure}[t]{0.24\linewidth}
    \includegraphics[width=\linewidth]{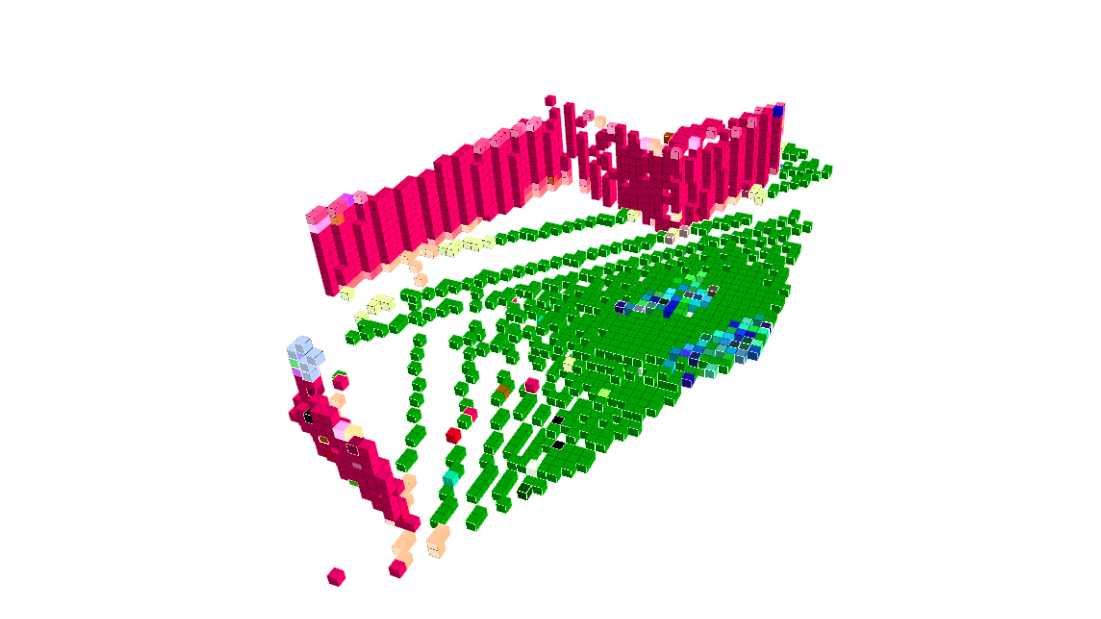}
    \caption{The robot begins exploration.}
    \end{subfigure}%
    \hfill%
    \begin{subfigure}[t]{0.24\linewidth}
    \includegraphics[width=\linewidth]{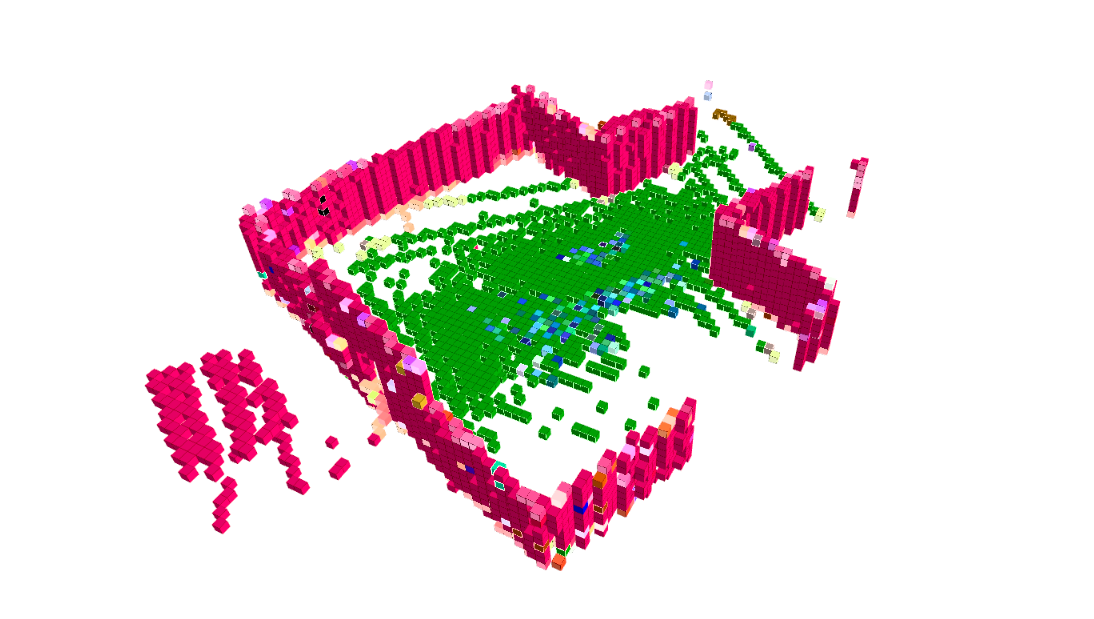}
    \caption{After $2$ iterations, the robot moves north to explore a larger unknown area.}
    \end{subfigure}%
    \hfill%
    \begin{subfigure}[t]{0.24\linewidth}
    \includegraphics[width=\linewidth]{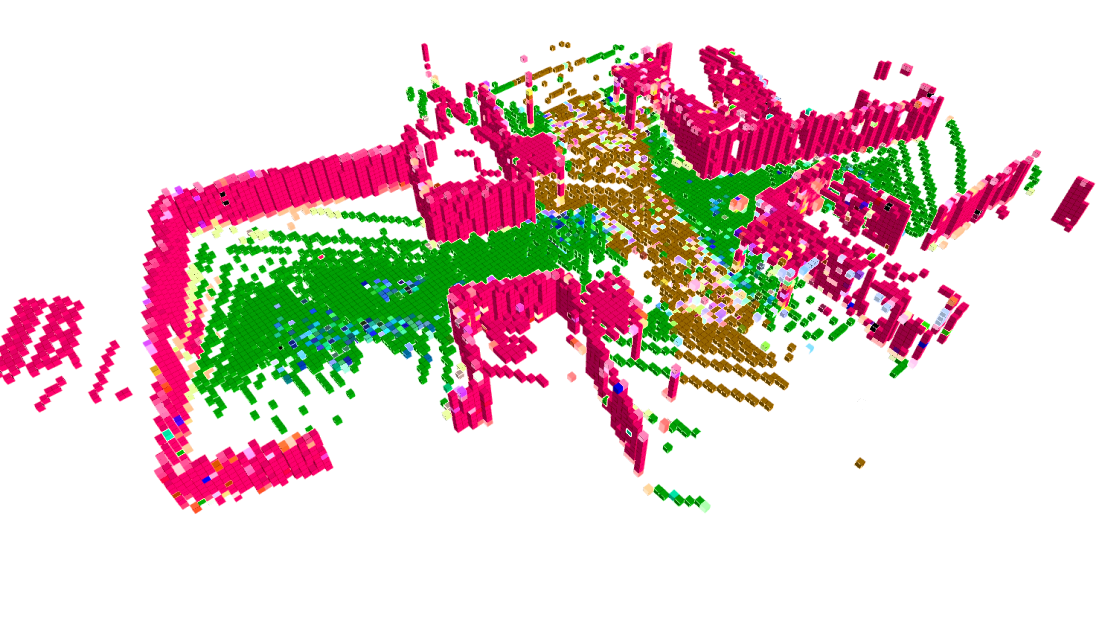}
    \caption{At iteration $8$, the robot continues exploration by visiting unknown sections in the north, east, and west.}
    \end{subfigure}%
    \hfill%
    \begin{subfigure}[t]{0.24\linewidth}
    \includegraphics[width=\linewidth]{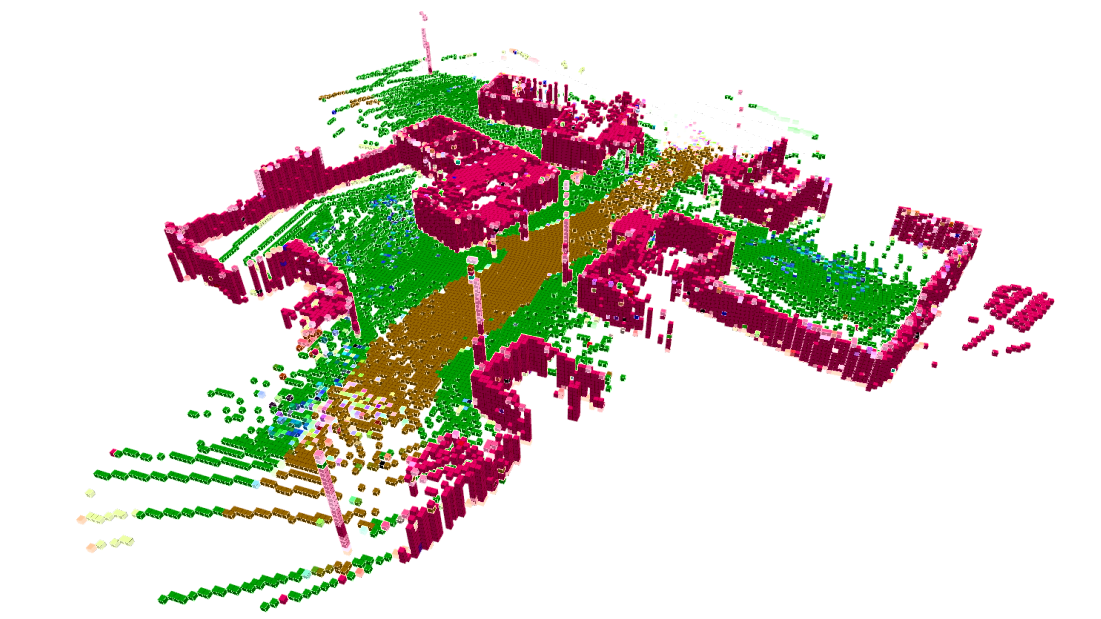}
    \caption{The robot begins to visit previously explored areas to fill partially observed objects at iteration $19$.}
    \end{subfigure}\\
    \begin{subfigure}[t]{0.49\linewidth}
    \centering
    \includegraphics[width=\linewidth]{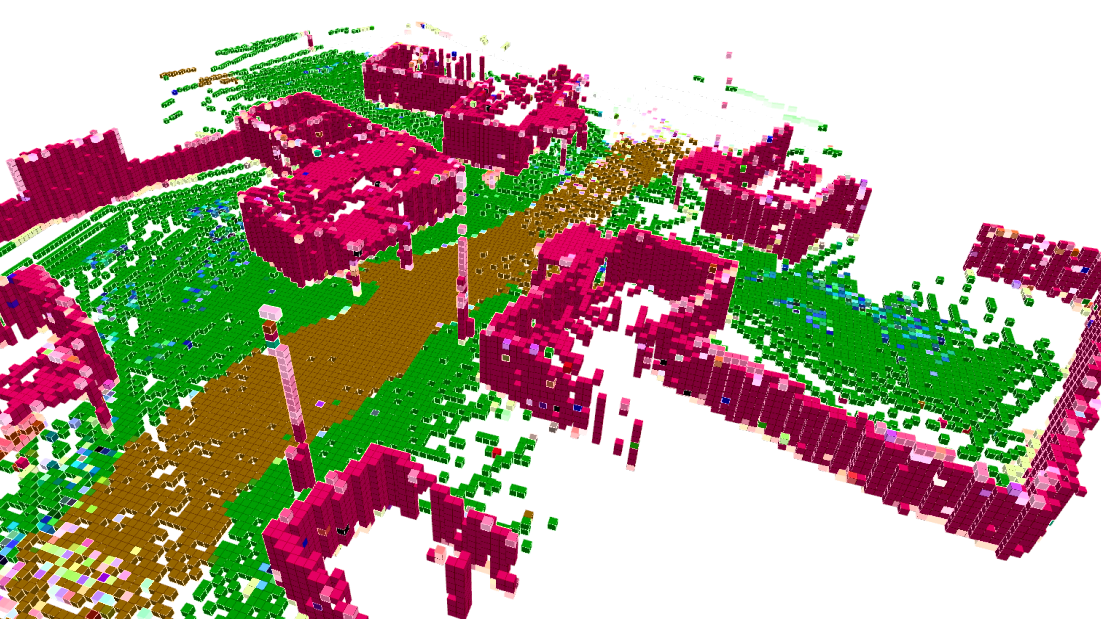}
    \caption{Multi-class occupancy map after $20$ exploration iterations}
    \end{subfigure}%
    \hfill%
    \begin{subfigure}[t]{0.49\linewidth}
    \centering
    \includegraphics[width=\linewidth]{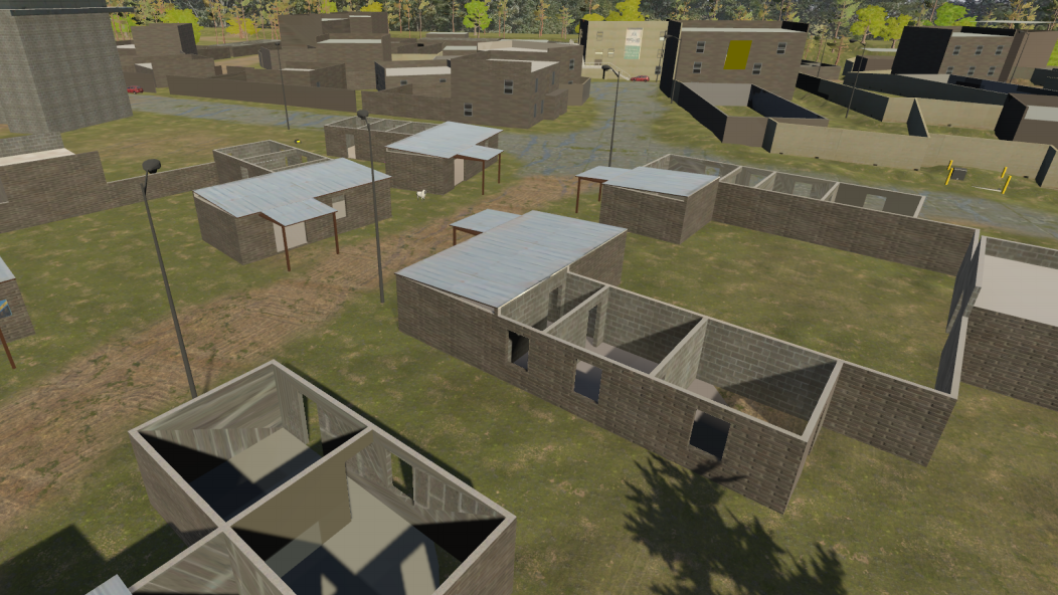}
    \caption{Photo-realistic Unity simulation environment}
    \end{subfigure}
    \caption{Time lapse of autonomous exploration and multi-class mapping in a simulated Unity environment. The robot is equipped with an RGBD sensor and runs semantic segmentation. Different colors represent different semantic categories (grass, dirt, building, etc.).}
    \label{fig:3d_sim_env}
\end{figure*}

\section{Conclusion}
\label{sec:conclusion}

This paper developed techniques for active multi-class mapping using range and semantic segmentation observations. Our results enable efficient mutual information computation over multi-class maps and make it possible to optimize for per-class uncertainty. Our experiments show that the proposed method performs on par with the state of the art FSMI method in binary active mapping scenarios. However, when semantic information is considered our method outperforms existing algorithms and leads to efficient exploration and accurate multi-class mapping. Our future work will develop an octree extension of the probabilistic multi-class grid map and closed-form mutual information computation, allowing our methods to scale to very large environments.


\appendices
\section{Proof of Prop.~\ref{prop:log-odds-bayes-rule}}
\label{app:log-odds-bayes-rule}

Applying Bayes rule in \eqref{eq:bayes_rule} and the factorization in \eqref{eq:pdf_factorization} to $p_t(\bfm)$ for some $\bfz \in \calZ_{t+1}$ leads to:
\begin{equation}
\scaleMathLine[0.89]{\prod_{i=1}^N p(m_i | \calZ_{1:t}, \bfz) = \frac{p(\bfz)}{p(\bfz | \calZ_{1:t})} \prod_{i=1}^N \frac{p(m_i | \bfz)}{p(m_i)} p(m_i | \calZ_{1:t}).}
\end{equation}
The term $\frac{p(\bfz)}{p(\bfz | \calZ_{1:t})}$ may be eliminated by considering the odds ratio of the cell probabilities:
\begin{equation}
\begin{aligned}
\prod_{i=1}^N &\frac{ p(m_i = k_i | \calZ_{1:t}, \bfz) }{ p(m_i = 0 | \calZ_{1:t}, \bfz) } \\
&= \prod_{i=1}^N \frac{p(m_i = k_i | \bfz)}{p(m_i = 0 | \bfz)} \frac{p(m_i = 0)}{p(m_i = k_i)} \frac{p(m_i = k_i | \calZ_{1:t})}{p(m_i = 0| \calZ_{1:t})}.
\end{aligned}
\end{equation}
Since each term in both the left- and right-hand side products only depends on one map cell $m_i$, the expression holds for each individual cell. Re-writing the expression for cell $m_i$ in vector form, with elements corresponding to each possible value of $k_i \in \calK$, and taking an element-wise log leads to:
\begin{equation}
\label{eq:13}
\begin{aligned}
\begin{bmatrix} \log \frac{p(m_i = 0 | \calZ_{1:t},\bfz)}{p(m_i = 0| \calZ_{1:t},\bfz)} & \cdots & \log \frac{p(m_i = K| \calZ_{1:t},\bfz)}{p(m_i = 0| \calZ_{1:t},\bfz)} \end{bmatrix}^\top\\
\qquad = (\bfl_i(\bfz) - \bfh_{0,i}) + \bfh_{t,i}
\end{aligned}
\end{equation}
Applying \eqref{eq:13} recursively for each element $\bfz \in \calZ_{t+1}$ leads to the desired result in \eqref{eq:log-odds-bayes-rule}.\qed

\section{Proof of Prop.~\ref{prop:mut_inf_semantic}}
\label{app:mut-inf-semantic}

Let $\calR_{t+1:t+T}(r_{max}) := \cup_{\tau,b} \calR_{\tau,b}(r_{max})$ be the set of map indices which can potentially be observed by $\underline{\calZ}_{t+1:t+T}$. Using the factorization in \eqref{eq:pdf_factorization} and the fact that Shannon entropy is additive for mutually independent random variables, the mutual information only depends on the cells whose index belongs to $\calR_{t+1:t+T}(r_{max})$, i.e.:
\begin{align}
    I(\bfm&; \underline{\calZ}_{t+1:t+T}  \mid \calZ_{1:t})\notag\\
    &= \sum_{\tau = t+1}^{t+T} \sum_{b = 1}^B \sum_{i \in \calR_{\tau,b}(r_{max})} I(m_i; \bfz_{\tau,b} \mid \calZ_{1:t}). \label{eq:mut_inf_decomp}
\end{align}
This is true because the measurements $\bfz_{\tau,b} \in \underline{\calZ}_{t+1:t+T}$ are independent by construction and the terms $I(m_i; \underline{\calZ}_{t+1:t+T} \mid \calZ_{1:t})$ can be decomposed into sums of mutual information terms between single-beam measurements $\bfz_{\tau,b}$ and the respective observed map cells $m_i$. The mutual information between a single map cell $m_i$ and a sensor ray $\bfz$ is equal to:
\begin{align}
  &I(m_i; \bfz \mid \calZ_{1:t}) = \label{eq:mut_inf_single_cell}\\
  &\scaleMathLine{\int p(\bfz \mid \calZ_{1:t}) \sum_{k=0}^K p(m_i = k \mid \bfz, \calZ_{1:t}) \log{\frac{p(m_i = k \mid \bfz, \calZ_{1:t})}{p_t(m_i = k)}} d\bfz.} \notag
\end{align}
Using the inverse observation model in \eqref{eq:log_inverse_observation_model} and the Bayesian multi-class update in \eqref{eq:log-odds-bayes-rule}, we have:
\begin{align}
    &\sum_{k=0}^K p(m_i = k \mid \bfz, \calZ_{1:t}) \log{\frac{p(m_i = k \mid \bfz, \calZ_{1:t})}{p_t(m_i = k)}} \notag\\ 
    &\scaleMathLine{ = (\bfl_i(\bfz) - \bfh_{0,i})^\top \sigma(\bfl_i(\bfz) - \bfh_{0,i} + \bfh_{t,i} ) +\log{\frac{p(m_i = 0 \mid \bfz, \calZ_{1:t})}{p_t(m_i = 0)}}}\notag\\
    & = f(\bfl_i(\bfz) - \bfh_{0,i}, \bfh_{t,i}), \label{eq:f_func}
\end{align}
where \eqref{eq:log_inverse_observation_model} and \eqref{eq:log-odds-bayes-rule} were applied a second time to the log term above. Plugging \eqref{eq:f_func} back into the mutual information expression in \eqref{eq:mut_inf_single_cell} and returning to \eqref{eq:mut_inf_decomp}, we have:
\begin{align}
\label{eq:mi-integral}
I(&\bfm; \underline{\calZ}_{t+1:t+T}  \mid \calZ_{1:t})\\
&= \sum_{\tau = t+1}^{t+T} \sum_{b = 1}^B \sum_{y = 1}^K \int_0^{r_{max}} \biggl( p(\bfz_{\tau,b} = (r,y) \mid \calZ_{1:t}) \notag\\
&\qquad\qquad\qquad \sum_{i \in \calR_{\tau,b}(r_{max})} \negquad f(\bfl_i((r,y)) - \bfh_{0,i}, \bfh_{t,i}) \biggr) dr. \notag
\end{align}
For $\bfz_{\tau,b}=(r,y)$, the second term inside the integral above can be simplified to:
\begin{equation}
\label{eq:C_cal}
\begin{aligned}
\Tilde{C}_{\tau,b}(r,y) &:= \negquad \sum_{i \in \calR_{\tau,b}(r_{max})} \negquad f(\bfl_i((r,y)) - \bfh_{0,i}, \bfh_{t,i})\\
&\phantom{:}= f(\bfphi^+ + \bfE_{y+1}\bfpsi^+ - \bfh_{0,i_{\tau,b}^*}, \bfh_{t,i_{\tau,b}^*})\\
&\qquad + \negquad\sum_{i \in \calR_{\tau,b}(r) \setminus \{i_{\tau,b}^*\}} \negquad f(\bfphi^- - \bfh_{0,i}, \bfh_{t,i})
\end{aligned}
\end{equation}
because for map indices $i \in \calR_{\tau,b}(r_{max}) \setminus \calR_{\tau,b}(r)$ that are not observed by $\bfz_{\tau,b}$, we have $\bfl_i((r,y)) = \bfh_{0,i}$ according to \eqref{eq:log_inverse_observation_model} and $f(\bfh_{0,i} - \bfh_{0,i}, \bfh_{t,i}) = 0$.

Next, we apply the approximation of \eqref{eq:cond_prob_approx} for the first term in the integral in \eqref{eq:mi-integral}; which leads to integration over $\tilde{p}_{\tau,b}(r,y) \tilde{C}_{\tau,b}(r,y)$ in \eqref{eq:mi-integral}. Note that $\Tilde{p}_{\tau,b}(r,y)$ and $\Tilde{C}_{\tau,b}(r,y)$ are piece-wise constant functions since $\calR_{\tau,b}(r)$ is constant with respect to $r$ as long as the beam $\bfz$ lands in cell $m_{i^*}$. Hence, we can partition the integration domain over $r$ into a union of intervals where the beam $\bfz$ hits the same cell, i.e. $\calR_{\tau,b}(r)$ remains constant:
\begin{equation*}
\int_{0}^{r_{max}} \negquad\Tilde{p}_{\tau,b}(r,y) \Tilde{C}_{\tau,b}(r,y)\,dr = \sum_{n=1}^{N_{\tau,b}} \int_{r_{n-1}}^{r_n} \negquad\Tilde{p}_{\tau,b}(r,y) \Tilde{C}_{\tau,b}(r,y)\,dr,
\end{equation*}
where $N_{\tau,b} = |\calR_{\tau,b}(r_{max})|$, $r_0 = 0$, and $r_N = r_{max}$. From the piece-wise constant property of $\Tilde{p}_{\tau,b}(r,y)$ and $\Tilde{C}_{\tau,b}(r,y)$ over the interval $(r_{n-1},r_n]$, one can easily obtain:
\begin{equation}
\begin{split}
    &\int_{r_{n-1}}^{r_n} \Tilde{p}_{\tau,b}(r,y) \Tilde{C}_{\tau,b}(r,y)\,dr \\ 
    &= \Tilde{p}_{\tau,b}(r_n,y) \Tilde{C}_{\tau,b}(r_n,y) \gamma(n) = p_{\tau,b}(n,y) C_{\tau,b}(n,y),
\end{split}
\end{equation}
where $p_{\tau,b}(n,y)$ and $C_{\tau,b}(n,y)$ are defined in the statement of Prop.~\ref{prop:mut_inf_semantic}. Therefore, substituting $y$ with $k$ and plugging the integration result into \eqref{eq:mi-integral} yields the lower bound for the mutual information between map $\bfm$ and observations $\calZ_{t+1:t+T}$ as in \eqref{eq:mut_inf_semantic}.\qed

\balance
{\small
\bibliographystyle{cls/IEEEtran}
\bibliography{main.bbl}

\begin{thebibliography}{10}
\providecommand{\url}[1]{#1}
\csname url@rmstyle\endcsname
\providecommand{\newblock}{\relax}
\providecommand{\bibinfo}[2]{#2}
\providecommand\BIBentrySTDinterwordspacing{\spaceskip=0pt\relax}
\providecommand\BIBentryALTinterwordstretchfactor{4}
\providecommand\BIBentryALTinterwordspacing{\spaceskip=\fontdimen2\font plus
\BIBentryALTinterwordstretchfactor\fontdimen3\font minus
  \fontdimen4\font\relax}
\providecommand\BIBforeignlanguage[2]{{%
\expandafter\ifx\csname l@#1\endcsname\relax
\typeout{** WARNING: IEEEtran.bst: No hyphenation pattern has been}%
\typeout{** loaded for the language `#1'. Using the pattern for}%
\typeout{** the default language instead.}%
\else
\language=\csname l@#1\endcsname
\fi
#2}}

\bibitem{occ_mapping_1}
A.~Elfes, ``Using occupancy grids for mobile robot perception and navigation,''
  \emph{Computer}, vol.~22, no.~6, pp. 46--57, 1989.

\bibitem{occ_mapping_2}
S.~Thrun, ``Learning occupancy grid maps with forward sensor models,''
  \emph{Autonomous Robots}, vol.~15, no.~2, pp. 111--127, 2003.

\bibitem{semantic_1}
S.~{Song}, F.~{Yu}, A.~{Zeng}, A.~X. {Chang}, M.~{Savva}, and T.~{Funkhouser},
  ``Semantic scene completion from a single depth image,'' in \emph{IEEE
  Conference on Computer Vision and Pattern Recognition (CVPR)}, 2017, pp.
  1746--1754.

\bibitem{semantic_2}
L.~{Gan}, R.~{Zhang}, J.~W. {Grizzle}, R.~M. {Eustice}, and M.~{Ghaffari},
  ``Bayesian spatial kernel smoothing for scalable dense semantic mapping,''
  \emph{IEEE Robotics and Automation Letters}, vol.~5, no.~2, pp. 790--797,
  2020.

\bibitem{semantic_3}
T.~{Wang}, V.~{Dhiman}, and N.~{Atanasov}, ``Learning navigation costs from
  demonstration with semantic observations,'' \emph{Proceedings of Machine
  Learning Research vol}, vol. 120, pp. 1--11, 2020.

\bibitem{bonnet}
A.~{Milioto} and C.~{Stachniss}, ``{Bonnet: An Open-Source Training and
  Deployment Framework for Semantic Segmentation in Robotics using CNNs},'' in
  \emph{{IEEE} International Conference on Robotics and Automation (ICRA)},
  2019.

\bibitem{julian}
B.~J. {Julian}, S.~{Karaman}, and D.~{Rus}, ``On mutual information-based
  control of range sensing robots for mapping applications,'' in
  \emph{{IEEE/RSJ} International Conference on Intelligent Robots and Systems
  (IROS)}, 2013, pp. 5156--5163.

\bibitem{csqmi}
B.~{Charrow}, S.~{Liu}, V.~{Kumar}, and N.~{Michael}, ``Information-theoretic
  mapping using cauchy-schwarz quadratic mutual information,'' in \emph{{IEEE}
  International Conference on Robotics and Automation (ICRA)}, 2015, pp.
  4791--4798.

\bibitem{fsmi}
Z.~{Zhang}, T.~{Henderson}, V.~{Sze}, and S.~{Karaman}, ``Fsmi: Fast
  computation of shannon mutual information for information-theoretic
  mapping,'' in \emph{{IEEE} International Conference on Robotics and
  Automation (ICRA)}, 2019, pp. 6912--6918.

\bibitem{frontier}
B.~{Yamauchi}, ``A frontier-based approach for autonomous exploration,'' in
  \emph{IEEE International Symposium on Computational Intelligence in Robotics
  and Automation}, 1997, pp. 146--151.

\bibitem{geo_exp_1}
W.~{Burgard}, M.~{Moors}, C.~{Stachniss}, and F.~E. {Schneider}, ``Coordinated
  multi-robot exploration,'' \emph{{IEEE} Transactions on Robotics (TRO)},
  vol.~21, no.~3, pp. 376--386, 2005.

\bibitem{geo_exp_2}
H.~H. Gonz{\'a}lez-Ba{\~n}os and J.-C. Latombe, ``Navigation strategies for
  exploring indoor environments,'' \emph{The International Journal of Robotics
  Research (IJRR)}, vol.~21, no. 10-11, pp. 829--848, 2002.

\bibitem{geo_exp_3}
C.~G{\'o}mez, M.~Fehr, A.~C. Hern{\'a}ndez, J.~Nieto, R.~Barber, and
  R.~Siegwart, ``{Hybrid Topological and 3D Dense Mapping through Autonomous
  Exploration for Large Indoor Environments},'' in \emph{{IEEE} International
  Conference on Robotics and Automation (ICRA)}, 2020.

\bibitem{geo_exp_4}
C.~Cao, J.~Zhang, M.~Travers, and H.~Choset, ``Hierarchical coverage path
  planning in complex 3d environments,'' in \emph{{IEEE} International
  Conference on Robotics and Automation (ICRA)}, 2020.

\bibitem{geo_exp_5}
R.~Maffei, M.~P. Souza, M.~Mantelli, D.~Pittol, M.~Kolberg, and V.~A.~M. Jorge,
  ``{Exploration of 3D terrains using potential fields with elevation-based
  local distortions},'' in \emph{{IEEE} International Conference on Robotics
  and Automation (ICRA)}, 2020.

\bibitem{info_exp_1}
A.~Elfes, ``Robot navigation: Integrating perception, environmental constraints
  and task execution within a probabilistic framework,'' in \emph{Reasoning
  With Uncertainty in Robotics}, 1995, pp. 93--130.

\bibitem{info_exp_2}
S.~J. {Moorehead}, R.~{Simmons}, and W.~L. {Whittaker}, ``Autonomous
  exploration using multiple sources of information,'' in \emph{{IEEE}
  International Conference on Robotics and Automation (ICRA)}, 2001, pp.
  3098--3103.

\bibitem{info_exp_3}
F.~{Bourgault}, A.~A. {Makarenko}, S.~B. {Williams}, B.~{Grocholsky}, and H.~F.
  {Durrant-Whyte}, ``Information based adaptive robotic exploration,'' in
  \emph{{IEEE/RSJ} International Conference on Intelligent Robots and Systems
  (IROS)}, 2002, pp. 540--545.

\bibitem{info_exp_4}
A.~Visser and B.~Slamet, ``Balancing the information gain against the movement
  cost for multi-robot frontier exploration,'' in \emph{European Robotics
  Symposium}, 2008, pp. 43--52.

\bibitem{info_exp_5}
T.~Kollar and N.~Roy, ``Efficient optimization of information-theoretic
  exploration in slam,'' in \emph{AAAI Conference on Artificial Intelligence},
  2008, p. 1369–1375.

\bibitem{active_slam_1}
H.~{Carrillo}, I.~{Reid}, and J.~A. {Castellanos}, ``On the comparison of
  uncertainty criteria for active slam,'' in \emph{IEEE International
  Conference on Robotics and Automation}, 2012, pp. 2080--2087.

\bibitem{active_slam_2}
L.~{Carlone}, J.~{Du}, M.~K. {Ng}, B.~Bona, and M.~Indri, ``{Active SLAM and
  exploration with particle filters using Kullback-Leibler divergence},''
  \emph{Journal of Intelligent \& Robotic Systems}, vol.~75, no.~2, pp.
  291--311, 2014.

\bibitem{active_slam_3}
N.~{Atanasov}, J.~{Le Ny}, K.~{Daniilidis}, and G.~J. {Pappas},
  ``{Decentralized active information acquisition: Theory and application to
  multi-robot SLAM},'' in \emph{IEEE International Conference on Robotics and
  Automation}, 2015, pp. 4775--4782.

\bibitem{active_slam_4}
J.~{Wang}, T.~{Shan}, and B.~{Englot}, ``Virtual maps for autonomous
  exploration with pose slam,'' in \emph{IEEE/RSJ International Conference on
  Intelligent Robots and Systems}, 2019.

\bibitem{info_exp_7}
C.~Stachniss, G.~Grisetti, and W.~Burgard, ``{Information gain-based
  exploration using Rao-Blackwellized particle filters},'' in \emph{Robotics:
  Science and Systems}, 2005, pp. 65--72.

\bibitem{info_exp_11}
T.~Henderson, V.~Sze, and S.~Karaman, ``An efficient and continuous approach to
  information-theoretic exploration,'' in \emph{{IEEE} International Conference
  on Robotics and Automation (ICRA)}, 2020.

\bibitem{ProbabilisticRoboticsBook}
S.~Thrun, W.~Burgard, and D.~Fox, \emph{Probabilistic Robotics}.\hskip 1em plus
  0.5em minus 0.4em\relax MIT Press Cambridge, 2005.

\bibitem{bresenham}
J.~E. {Bresenham}, ``Algorithm for computer control of a digital plotter,''
  \emph{IBM Systems Journal}, vol.~4, no.~1, pp. 25--30, 1965.

\bibitem{octomap}
A.~Hornung, K.~M. Wurm, M.~Bennewitz, C.~Stachniss, and W.~Burgard, ``{OctoMap:
  an efficient probabilistic 3D mapping framework based on octrees},''
  \emph{Autonomous Robots}, vol.~34, pp. 189--206, 2013.

\end{thebibliography}
}

\end{document}